\documentclass{article}
\usepackage{ijcai23}

\usepackage{times}
\usepackage{soul}
\usepackage{url}
\usepackage[hidelinks]{hyperref}
\usepackage[utf8]{inputenc}
\usepackage[small]{caption}
\usepackage{graphicx}
\usepackage{amsmath}
\usepackage{amsthm}
\usepackage{booktabs}
\usepackage[ruled,vlined]{algorithm2e}
\usepackage[switch]{lineno}
\usepackage{wrapfig}
\usepackage{sidecap}
\usepackage{subcaption}
\usepackage{graphicx}
\usepackage{multirow}
\usepackage[table]{xcolor}

\newtheorem{theorem}{Theorem}

\title{Learning Prototype Classifiers for Long-Tailed Recognition}

\author{
Saurabh Sharma$^1$
\and
Yongqin Xian$^2$\and
Ning Yu$^{3}$\And
Ambuj Singh$^1$
\affiliations
$^1$University of California Santa Barbara, USA\\
$^2$Google, Switzerland\\
$^3$Salesforce Research, USA\\
\emails
\{saurabhsharma, ambuj\}@cs.ucsb.edu,
yxian@google.com,
ning.yu@salesforce.com
}

\begin{document}

\maketitle

\begin{abstract}
The problem of long-tailed recognition (LTR) has received attention in recent years due to the fundamental power-law distribution of objects in the real-world. Most recent works in LTR use softmax classifiers that are biased in that they correlate classifier norm with the amount of training data for a given class. In this work, we show that learning prototype classifiers addresses the biased softmax problem in LTR. Prototype classifiers can deliver promising results simply using Nearest-Class-Mean (NCM), a special case where prototypes are empirical centroids. We go one step further and propose to jointly \emph{learn prototypes} by using distances to prototypes in representation space as the logit scores for classification. Further, we theoretically analyze the properties of Euclidean distance based prototype classifiers that lead to stable gradient-based optimization which is robust to outliers. To enable independent distance scales along each channel, we enhance Prototype classifiers by learning channel-dependent temperature parameters. Our analysis shows that prototypes learned by Prototype classifiers are better separated than empirical centroids. Results on four LTR benchmarks show that Prototype classifier outperforms or is comparable to state-of-the-art methods. Our code is made available at \href{https://github.com/saurabhsharma1993/prototype-classifier-ltr}{\textit{https://github.com/saurabhsharma1993/prototype-classifier-ltr}}.
\end{abstract}

\section{Introduction}
\label{Introduction}

Imbalanced datasets are prevalent in the natural world due to the fundamental power-law distribution of objects~\cite{van2017devil}. Past decades have seen a lot of research in class-imbalanced learning~\cite{estabrooks2004multiple,he2009learning}, and more recently, due to its far-reaching relevance in the era of deep-learning, the problem of long-tailed recognition (LTR) has received significant attention~\cite{cui2019class,cao2019learning,liu2019large} in the field of computer vision. 

A common underlying assumption of recent work in LTR is that softmax classifiers learned from imbalanced training datasets are biased towards head classes. They seek to rectify head class bias in the classifier through data-resampling~\cite{kang2019decoupling}, loss reshaping~\cite{ren2020balanced,menon2020long,samuel2021distributional}, ensemble-based models~\cite{xiang2020learning,wang2020long,zhou2020bbn}, and knowledge transfer~\cite{liu2019large,liu2021gistnet}. Particularly, it was shown in~\cite{kang2019decoupling} that deep models trained on imbalanced datasets have a high correlation of classifier norm to class size. To achieve similar classifier norms, they freeze the feature extractor and retrain the classifier with balanced class sampling in a second stage. Alternatively, they normalize every classifier by its $L_2$ norm or a power of its norm, leading to balanced classifer norms. More recently,~\cite{alshammari2022long} showed that balanced weights in both the feature extractor and classifier can be achieved by regularizing the model with higher weight decay and/or MaxNorm constraints.  

\begin{figure*}[!ht]{}
    \includegraphics[width=\textwidth]{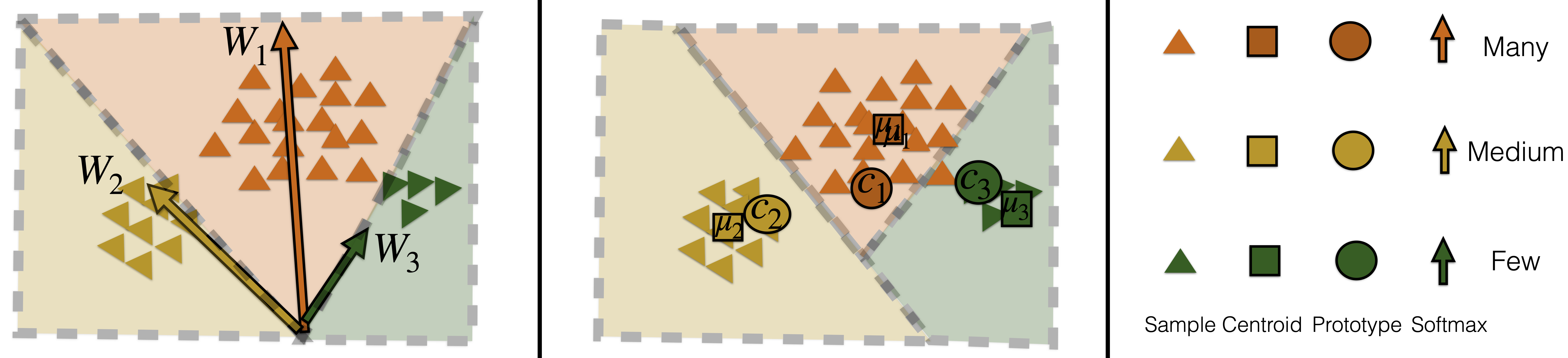}
    \caption{An illustration of Softmax vs Prototype classifiers for long-tailed data. Softmax classifiers have both a direction and a magnitude, indicated by the orientation and length of the classifier vector. During imbalanced training, the length of softmax classifiers gets correlated to the class size and leads to classification boundaries biased towards classes with Many samples. In contrast, prototype classifiers circumvent this shortcoming by using Euclidean distances to learnable prototypes within representation space, leading to more fair decision boundaries.}
    \label{fig:figure 1}
\end{figure*}

The shortcoming of correlating softmax classifier norm to class size is directly connected to the classification model. We show an illustration in Fig.~\ref{fig:figure 1}. Consider the problem of learning a classifier $f_W$ parameterized by weight matrix $W$. We assume the feature backbone $g_\theta$ is fixed and known. The classifier outputs logit scores $f_W(g(x)) = W^Tg(x)$, and the posterior probability $p(y|g(x)) \propto p(y)p(g(x)|y) \propto (W^Tg(x))_y$. Note that softmax classifier has both a magnitude and direction associated to it. In the ideal case, the directions get aligned with the class means $\mu_y$ to encode $p(g(x)|y)$~\cite{papyan2020prevalence,thrampoulidisimbalance}. But for an imbalanced training dataset, to encode the prior of the label distribution $p(y)$ into classifier weights, the learned classifier norm gets correlated to the class size.

To circumvent this shortcoming of softmax in the long-tailed classification setting, it is advantageous to change the classification model from dot-product to distances in the representation space. More specifically, associate a learnable prototype $c_y$ to each class $y$, and let $p(y|x) \propto -d(g(x),c_y)$, where $d(\cdot,\cdot)$ is a distance metric in the representation space. With a suitable choice of distance metric, such as Euclidean or cosine distances, we obtain a simple and effective prototype classifier which avoids the biased classifier norm problem. While prior work in long-tailed recognition explores the closely related Nearest-Class-Mean classifier \cite{kang2019decoupling}, which is a special case of our model with fixed prototypes, the full potential of prototype based classifiers has not been fully explored. In this work we thoroughly analyze prototype classifiers and show that they offer a promising alternative to softmax in long-tailed recognition.

We theoretically analyze and show that the Euclidean distance based prototype classifier has stable gradient-based optimization properties for learning prototypes. In particular, the L2 norm of the gradient on a prototype $c_y$ from a point $x$ is always equal in magnitude to the misclassification probability. Further, the gradient is oriented along the line joining $c_y$ to $x$. This leads to stable gradient updates that aren't affected by outliers that may be far from the prototypes. 

We also allow for variable scales along the different dimensions of the feature representations in the prototype distance. This is enabled by learning channel dependent temperature parameters in the prototype distance function. We find that this is beneficial since it's a means of learnable feature scaling and/or selection in the prototype classifier. 
It also lets the model learn a general Mahalanobis distance metric parameterized by a diagonal matrix. 

Overall, here is a summary of our main contributions: 

(1) We propose a novel learnable prototype classifier for long-tailed recognition to counter the shortcoming of softmax to correlate classifier norm to class size. This offers a promising distance-based alternative to dot-product classifiers that is relatively underexplored in the long-tailed setting. 

(2) We theoretically analyze the properties of Euclidean distance based prototype classifiers that leads to stable gradient-based optimization which is robust to outliers.  

(3) We enhance prototype classifiers with channel dependent temperature parameters in the prototype distance function, allowing for learnable feature scaling and/or selection. 

(4) We evalute our proposed prototype classifier on four benchmark long-tailed datasets, CIFAR10-LT, CIFAR100-LT, ImageNet-LT and iNaturalist18. Our results show that the prototype classifier outperforms or is comparable to the recent state-of-the-art methods. 

\section{Related Work}
\label{Related Work}

\paragraph{Prototype-based classifiers.} Prototype and Nearest-Class-Mean classifiers appear prominently in metric learning and few-shot learning. \cite{mensink2013distance} proposed learning a Mahalanobis distance metric on top of fixed representations, which leads to a NCM classifier based on the Euclidean distance in the learned metric space. \cite{snell2017prototypical} learn deep representations for few shot learning by minimizing the prototype based loss on task-specific epsiodes. \cite{guerriero2018deepncm} learn deep representations by optimizing the NCM classification objective on the entire training data. \cite{rebuffi2017icarl} learn NCM classifiers by maintaining a fixed number of samples to compute prototypes or centroids. In this work, we neither learn a distance metric or finetune representations using the NCM objective; instead our Prototype classifier directly learns the prototypes from pre-trained representations for long-tailed recognition.

\paragraph{Long-tailed recognition.} The literature on long-tailed recognition is quite vast. Data resampling or reweighting is the most commonplace strategy against imbalanced datasets. This generally takes three forms: (i) Oversampling minority class samples by adding small perturbations to the data~\cite{chawla2002smote,chawla2003smoteboost}, (ii) Undersampling majority class samples by throwing away some data~\cite{drummond2003c4}, and (iii) Uniform or class-balanced sampling based on the number of samples in a class~\cite{sun2019meta,xian2019f}. However, oversampling minority class samples has been shown to lead to overfitting and undersampling majority class samples can cause poor generalization~\cite{he2009learning}. Recently, ~\cite{kang2019decoupling} learn the representations using instance-based sampling and retrain the softmax classifier in a second step using uniform sampling while keeping underlying representations fixed. We also use the two-stage training strategy in our work; however, we avoid learning an additional softmax classifer and instead use Prototype classifiers as it operates in the representation space directly. Another line of prior work focuses on engineering loss functions suited for imbalanced datasets. Focal loss~\cite{lin2017focal} reshapes the standard cross-entropy loss to push more weight on misclassified and/or tail class samples. Class-balanced loss~\cite{cui2019class} uses a theoretical estimate of the volume occupied by a class to reweigh the loss function. \cite{cao2019learning} offset the label's logit score with a power of class size to achieve a higher decision boundary margin for tail classes. \cite{ren2020balanced} apply the offset to all the logits, and change the class size exponent from $1/4$ in prior work to $1$ and use a meta-sampler to retrain the classifier. \cite{menon2020long} apply a label-dependent adjustment to the logit score before softmax which is theoretically consistent. In our work, we build upon the logit adjustment approach and show how it can also be adapted to Prototype classifiers. Many other papers employ a specialized ensemble of experts to reduce tail bias and model variance. \cite{sharma2020long} train experts on class-balanced subsets of the training data and aggregate them using a joint confidence calibration layer. \cite{xiang2020learning} also train experts on class-balanced subsets and distill them into a unifed student model. \cite{wang2020long} explicitly enforce diversity in experts and aggregate them using a routing layer. \cite{zhou2020bbn} train two experts using regular and reversed sampling together with an adaptive learning strategy. In other works, self-supervised contrastive pretraining is used to improve feature representations for long-tailed recognition \cite{samuel2021distributional,cui2021parametric}. More recently, \cite{alshammari2022long} show that quite good performance on long-tailed datasets is possible by simply tuning weight decay properly. Thus, we build our Prototype classifiers on top of the representation learning backbone provided by them. However, the gain of Prototype classifiers over Softmax is invariant to the choice of the backbone, as we show in our experiments. 

\section{Our Approach}
\label{Section4}

\subsection{Prototype Classifier} 

The Prototype classifier assigns probability scores to classes based on distances to learnable prototypes in representation space. We assume that the representations $g(x)$ are fixed. In other words, the weights $\theta$ of the feature extractor $g_\theta$ are learned in a prior stage and fixed to enable learning of prototype classifiers. The decoupling of representation and classifier learning is commonly used to train softmax classifiers in long-tailed recognition \cite{kang2019decoupling}; here we adopt the same strategy to learn prototype classifiers. 

Let the learnable prototype for class $y$ be denoted by $c_y$. Under the assumption that the class-conditional distribution $p(g(x)|y)$ is a Gaussian with unit variance centered around the prototype $c_y$, we get the following probabilistic model for the class posterior $p(y|g(x))$:  
\begin{align}
\begin{split}
\label{eq:1}
  \log p(y|g(x)) &\propto -\frac{1}{2}d(g(x),c_y) \\
  &= \log{\frac{e^{-\frac{1}
  {2}d(g(x),c_y) }}{\sum_{y'} e^{-\frac{1}{2}d(g(x),c_y')}}}
\end{split}
\end{align}  

This leads to the nearest class prototype decision function for classification:
\begin{align*}
y^{*} &= \underset{{y \in \{1,\dots,\Bar{y}\}}}{\text{argmin}} d(g(x),c_y)   
\end{align*}
Different choices of distance metric lead to different prototype classifiers. While squared Euclidean distance and Euclidean distance lead to equivalent decision boundaries, we find that Euclidean distance is more amenable for learning of the prototypes due to its stable gradient. The Prototype classifier with Euclidean distance metric specifically uses:
\begin{align}
\label{eq:2}
    d(g(x),y) &= \sqrt{(g(x) - c_y)^T(g(x) - c_y)} 
\end{align}
The prototypes are learned using the Maximum Likelihood Estimation principle by minimizing the Negative Log Likelihood objective function:
\begin{align}
\label{eq:3}
L &= -\frac{1}{N} \sum\limits_{i}^N \log p(y_i|g(x_i))  
\end{align}
The remainder of this section is organized as follows: in Section~\ref{gradient updates}, we show how Euclidean distance leads to stable gradient updates on prototypes, then in Section~\ref{biased softmax} we show how learning prototypes addresses the biased softmax problem, followed by Section~\ref{cdt} which details our channel-dependent temperatures to enhance prototype learning, and Section~\ref{la} which describes logit adjustment for learning prototypes. Our entire pipeline in presented in Algorithm~\ref{alg: algorithm1}.

\subsection{Stable Gradient Updates on Prototypes}
\label{gradient updates}
In Prototype classifier, we use the Euclidean distance in Eq.~\ref{eq:2} due to its stable gradient optimization. We now describe the gradient updates on the prototypes. 

\begin{theorem}
    Let $L_{xy}$ denote the NLL loss on a single sample $x$ with label $y$. Then, the $L_2$ norm of the gradient on prototype $c_z$ is given by, 
    \[ \|\frac{\partial L_{xy}}{c_z}\|_2 = \left\{
  \begin{array}{lr}
    1 - p(y|g(x)) & : z = y \\
    p(z|g(x)) & : z \neq y 
  \end{array}
\right.
\]
Further,  if $z = y$, then the gradient descent direction lies along $g(x) - c_z$, else it is exactly the opposite direction.
\end{theorem}

\begin{proof}
Let $Z(x) = {\sum_{y'} e^{-\frac{1}{2}d(g(x),c_y')}}$ denote the normalization term over all the classes, and $\delta_{yz}$ denotes the Kronecker delta. Then, we get a closed form expression for the gradient using repeated application of the chain rule,
\begin{align*}
\begin{split}
\frac{\partial L_{xy}}{\partial c_z} &= \frac{\partial 
(-\log{\frac{e^{-\frac{1}{2}d(g(x),c_y)}}{Z(x)}})}{\partial c_z} \\
&= \frac{\partial 
(-\log{\frac{e^{-\frac{1}{2}\sqrt{c_y^Tc_y - 2g(x)^Tc_y + g(x)^Tg(x)}}}{Z(x)}})}{\partial c_z} \\
        &= \frac{\delta_{yz}(1-p(z|g(x)))(c_z-g(x))}{d(g(x),c_z)} \\
        &+ \frac{(1-\delta_{yz})p(z|g(x))(g(x)-c_z)}{d(g(x),c_z)}
\end{split}
\end{align*}
\end{proof}

From the gradient update, we see that if sample $x$ belongs to class $y$, then the updated prototype is shifted in the direction ($g(x)-c_z$), weighted by the  misclassification probability $(1-p(z|x))$, otherwise it's shifted in the opposite direction weighted by the misclassification probability $p(z|x)$. 

In both cases, the $L_2$ norm of the gradient is equal to the misclassification probability, leading to stable gradient updates. Importantly, it is invariant to $d(g(x),c_z)$, thus both inliers and outliers contribute equally. On the other hand, if we use squared Euclidean distance, it is easy to see that the $L_2$ norm of the gradient is scaled by $d(g(x),c_z)$, which is highly sensitive to outliers. We also verify empirically in Section~\ref{ablations} that Euclidean distance outperforms squared Euclidean distance for learning prototypes.  

\subsection{Addressing the Biased Softmax Problem}
\label{biased softmax}
We show how learning prototypes addresses the biased softmax problem by drawing a connection with the softmax classifier. We consider below a Prototype classifier based on the squared Euclidean distance metric: 

\begin{align*}
-\frac{1}{2}d(g(x),c_y) &= -\frac{1}{2}(g(x) - c_y)^T(g(x)-c_y) \\
&= -\frac{1}{2}(c_y^Tc_y - 2g(x)^Tc_y + g(x)^Tg(x))
\end{align*}

Since the term $-\frac{1}{2}g(x)^Tg(x)$ is shared by all classes, it can be ignored. Thus, the Prototype classifier with squared Euclidean distance metric is a softmax classifier with weight $c_y$ and bias $-\frac{1}{2}c_y^Tc_y$. 
Now, we know that softmax can increase or decreases the norm of its weight terms to account for an imbalanced prior label distribution $p(y)$~\cite{kang2019decoupling}. However, the same is not true for the Prototype classifier, because the bias term negates the gains from increasing or decreasing the norm of the weight term. Due to this coupling of the weight and bias terms, the Prototype classifier is more robust to imbalanced distributions and learns equinorm prototypes. We further show in our experiments in Section~\ref{prototype norms} that the learned prototypes
are actually very similar in norm. 

\subsection{Channel-Dependent Temperatures}
\label{cdt}
While computing distances for prototype classifiers, it is worthwhile to consider that all dimensions in the representation space may not have the same distance scale. Specifically, for a large distance along a dimension with high variance may not mean as much a large distance along a dimension with low variance. Thus, we define the channel-dependent temperature (CDT) prototype distance function using learnable temperatures $T_i$,
\begin{align}
\label{eq:4}
    d_{CDT}(g(x),y) &= \sqrt{\sum\limits_{i=1}^d\frac{(g(x) - c_y)_i^2}{T_i}}
\end{align}
A high value of temperature $T$ implies low sensitivity to distances along that dimension, and a low value implies high sensitivity. Learnable temperatures allow the model to learn feature scaling in the distance function. By letting the temperature approach very high values the model can also filter out potentially harmful features, i.e, feature selection. Finally, the CDT distance can also be interpreted as a Mahalanobis distance $\sqrt{(g(x)-c_y)^T\Sigma^{-1}(g(x)-c_y)}$, where $\Sigma$ is a diagonal matrix with the temperatures on the diagonal. 

\subsection{Logit Adjustment for Learning Prototypes}
\label{la}
Recent work by ~\cite{menon2020long} suggests that, for a classifier trained on an imbalanced dataset, it is important to perform an additive adjustment to the logits based on the distribution prior. Motivated by this strategy, we modify the prototype classifier loss function in Eq.~\ref{eq:3} as follows:
\begin{align}
    \label{eq:5}
    L_{xy} = -\log{\frac{e^{-\frac{1}{2}d(g(x),c_y) + \tau \cdot \log{N_y}}}{\sum_{y'} e^{-\frac{1}{2}d(g(x,c_y') + \tau \cdot \log{N_{y'}}}}}
\end{align}
Here $N_y$ denotes the number of training samples for class $y$ and $\tau$ is a hyperparameter. Logit adjustment is only done during training and not during inference. It places a higher penalty on misclassifying tail class samples relative to head classes. Empirical results show that logit adjustment does yield better results for learning prototype classifiers.

\begin{algorithm}[t!]
\SetKwInput{Otp}{Output}
\SetKwInput{Itp}{Input}
\SetKwInput{Require}{Require}
\SetKwInOut{Return}{return}
\Require{Pretrained feat reps $g(x)$, class means $\mu_y$}
\Require{Class-sizes $N_y$, logit-adj weight $\tau$, lrs $\alpha, \beta$}
Initialize prototypes $c_y = \mu_y$ and temperatures $T_i=1$\;
\For{$j\gets0$ \KwTo num\_batches}{
    Sample class-balanced batch $(g(x_k),y_k)_{k=1}^{batchsize}$\;
    $\forall{k}$ \& $\forall{y'}$ compute $d_{CDT}(g(x_k),c_{y'})$ using Eq.~\ref{eq:4}
    Compute $L = \sum_k{L_{x_ky_k}}$ using Eq.~\ref{eq:5}\;
    $\forall{y'}$ update $c_{y'} \gets c_{y'} - \alpha\nabla_{c_{y'}}L$\;
    $\forall{T_i}$ update $T_i \gets T_i - \beta\nabla_{T_i}L$\;
}
\caption{Learning prototype classifier}
\label{alg: algorithm1}
\end{algorithm}

\section{Experiments}
\label{Section5}
\subsection{Datasets}
We evaluate our proposed method on the following three benchmark long-tailed datasets:
1. \textbf{CIFAR100-LT} \cite{cao2019learning}: CIFAR100 consists of 60K images from 100 classes. Following prior work, we exponentially decay number of training samples and control the degree of data imbalance with an imbalance factor $\beta$. $\beta = N_{max}/N_{min}$, where $N_{max}$ and $N_{min}$ are the maximum and minimum number of training images per class respectively. $N_{max}$ is kept fixed at 500, and we report on three commonly used imbalance ratios $\beta \in [100, 50, 10]$. The validation set consists of 100 images per class and is also used as the test set.
2. \textbf{ImageNet-LT} \cite{liu2019large}: This is a long-tailed split of ImageNet. ImageNet-LT has an imbalanced training set with 115,846 images for 1,000 classes from ImageNet-1K~\cite{deng2009imagenet}. The class frequencies follow a natural power-law distribution~\cite{van2017devil} with a maximum of 1,280 and a minimum of 5 images per class. The validation and testing sets are balanced and contain 20 and 50 images per class respectively.
3. \textbf{iNaturalist-LT} \cite{van2018inaturalist}: This is a species classification dataset that naturally follows a long-tailed distribution. There are 8,142 classes, with a maximum of 118,800 images per class and a minimum of 16. The validation set contains 3 images per class and is also used as the test set. 
\subsection{Evaluation Protocol}
To distinguish the performance of the model on head vs tail classes, we report average top-1 accuracy on balanced test sets across four splits, \emph{Many}: classes with $\geq100$ samples, \emph{Med}: classes with $20\sim100$ samples, \emph{Few}: classes $<20$ samples, and \emph{All} classes. Since the test sets are balanced, average accuracy and mean precision are the same. 

\subsection{Implementation Details}
We train prototype classifiers on fixed representations from a prior pre-trained model. Our training pipeline consists of two stages, representation learning followed by prototype learning. To compare fairly with prior work, we use the same backbone architectures: ResNet-32 for CIFAR100, ResNeXt50 for ImageNet-LT and ResNet-50 for iNaturalist18. 

\paragraph{Representation learning.} We train all models end to end with an auxilliary softmax classifier by minimizing the average cross-entropy classification loss. At the end of training we discard the softmax classifier. We use instance-balanced data sampling \cite{kang2019decoupling} and regularize model weights using high weight decay \cite{alshammari2022long}, which achieves balanced weights in the feature extractor.  We use the SGD optimizer with momentum 0.9, learning rate 0.01 along with cosine learning rate schedule \cite{loshchilovsgdr} decaying to 0 for all models. For CIFAR100-LT and ImageNet-LT we train for 200 epochs with batch size 64 and for iNaturalist18 we train for 200 epochs with batch size 512. The weight decay parameter is found using Bayesian hyperparameter search \cite{nogueira2014bayesian}. We use random cropping and horizontal flipping for training data augmentation. All our training is done on NVIDIA V100 GPUs. 

\paragraph{Prototype learning.} We freeze representations and compute empirical centroids to initialize the class prototypes. We minimize the average logit-adjusted cross entropy loss with logit-adjustment weight $\tau=1/4$. We train for 1 epoch with class-balanced sampling, set prototype learning rate 4, CDT learning rate 0.005 and use the SGD optimizer with momentum 0.9. We also use random cropping, horizontal flipping, Autoaug~\cite{cubuk2018autoaugment} and Cutout~\cite{devries2017cutout} for data augmentation.  

\begin{figure*}[!t]
    \centering
    \begin{subfigure}[t]{0.245\textwidth}
        \includegraphics[width=\textwidth]{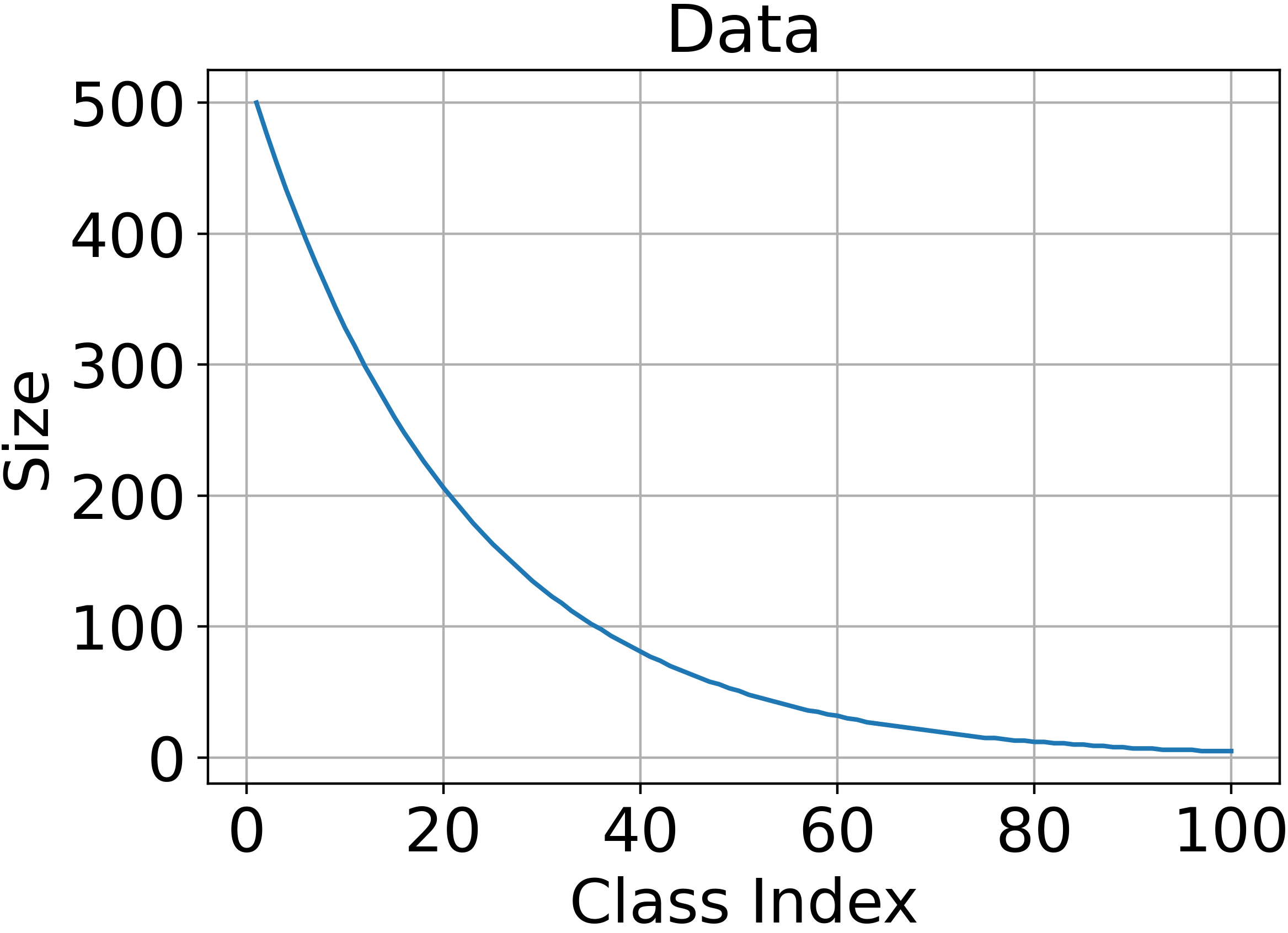}
        \caption{}
        \end{subfigure}
      \begin{subfigure}[t]{0.245\textwidth}
        \includegraphics[width=\textwidth]{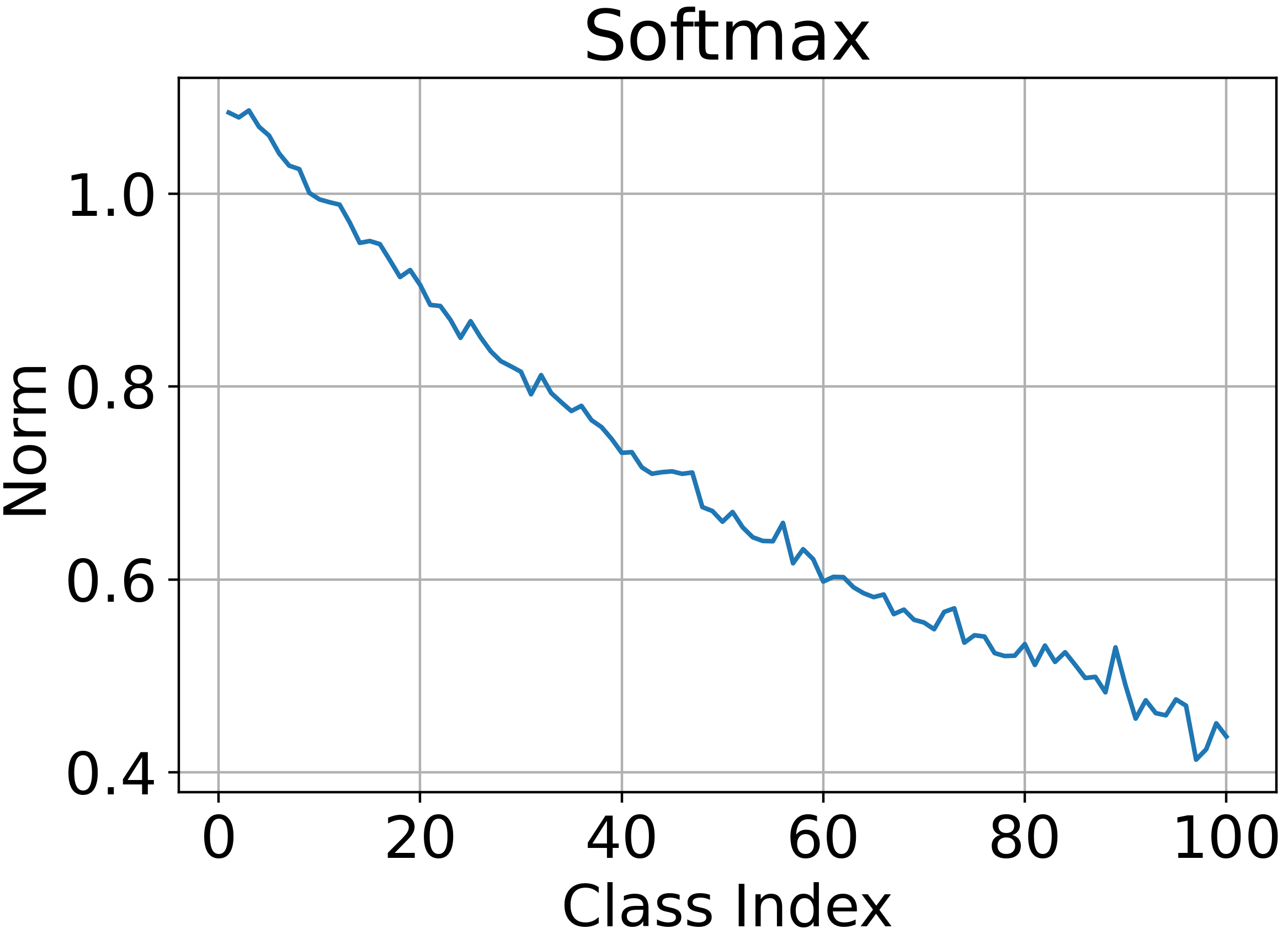}
        \caption{}
        \end{subfigure}
      \begin{subfigure}[t]{0.238\textwidth}
        \includegraphics[width=\textwidth]{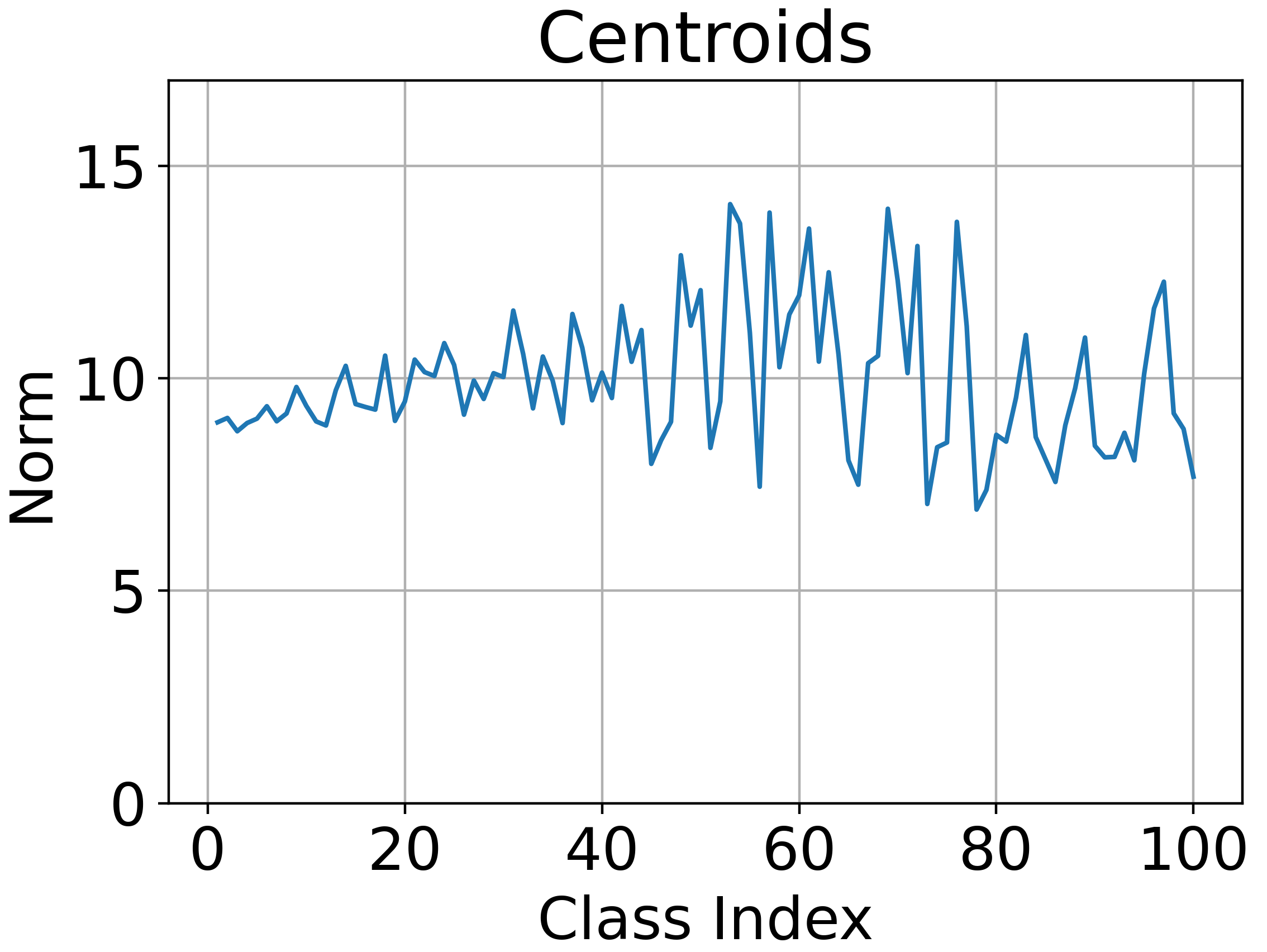}
        \caption{}
        \end{subfigure}
    \begin{subfigure}[t]{0.238\textwidth}
        \includegraphics[width=\textwidth]{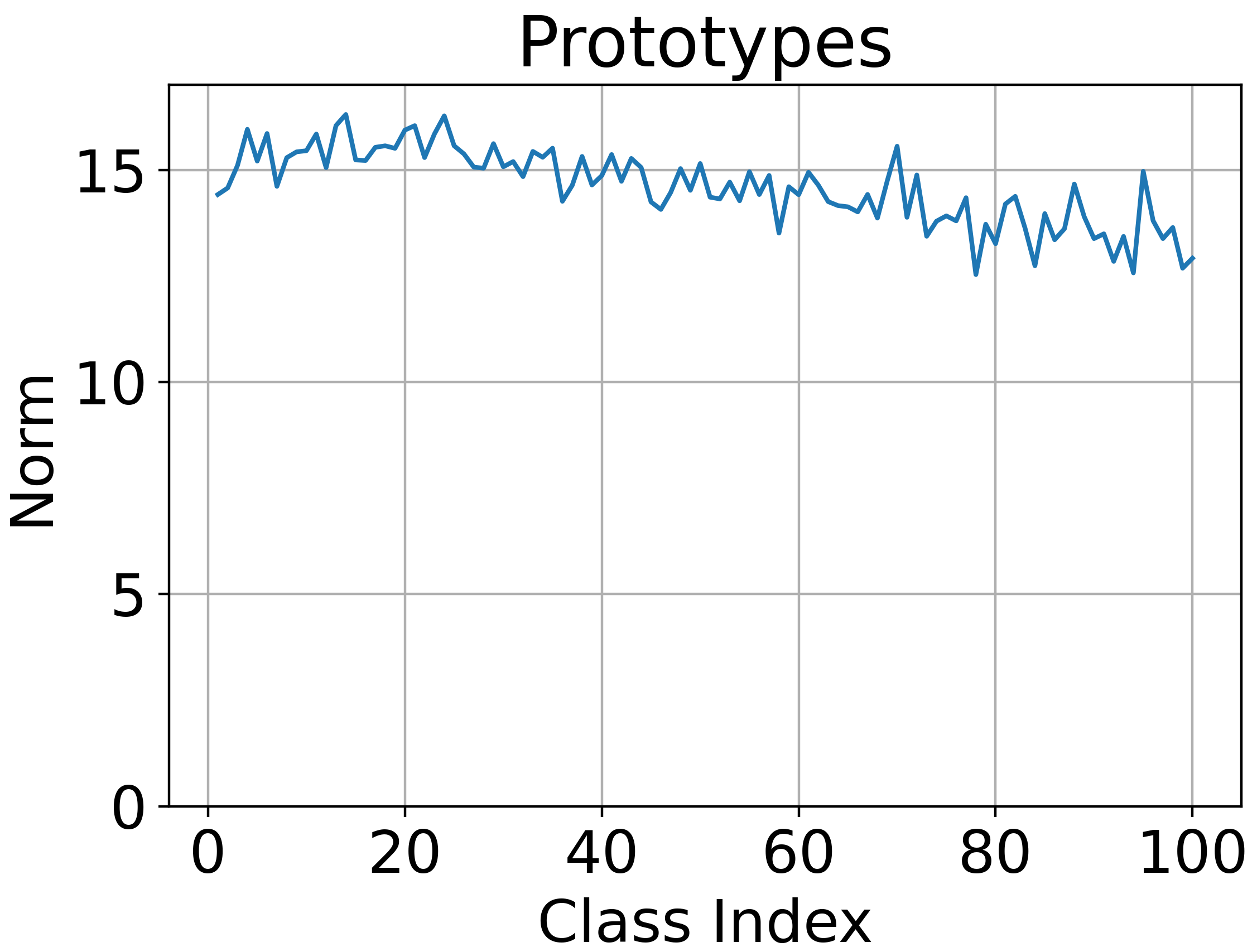}
        \caption{}
        \end{subfigure}
   \caption{Comparison of softmax and prototype norms. L-R: a) Class size vs index, b),c) and d): Norm of softmax classifer, centroids and prototypes vs class index respectively on CIFAR100-LT. Classes are sorted according to decaying class size. Prototype classifier achieves balanced norms, in contrast to the biased softmax classifier and the imbalanced centroid norms.} 
    \label{fig:fig2}
\end{figure*}

\begin{figure*}[!t]
    \centering
    \begin{subfigure}[t]{0.49\textwidth}
        \includegraphics[width=\textwidth]{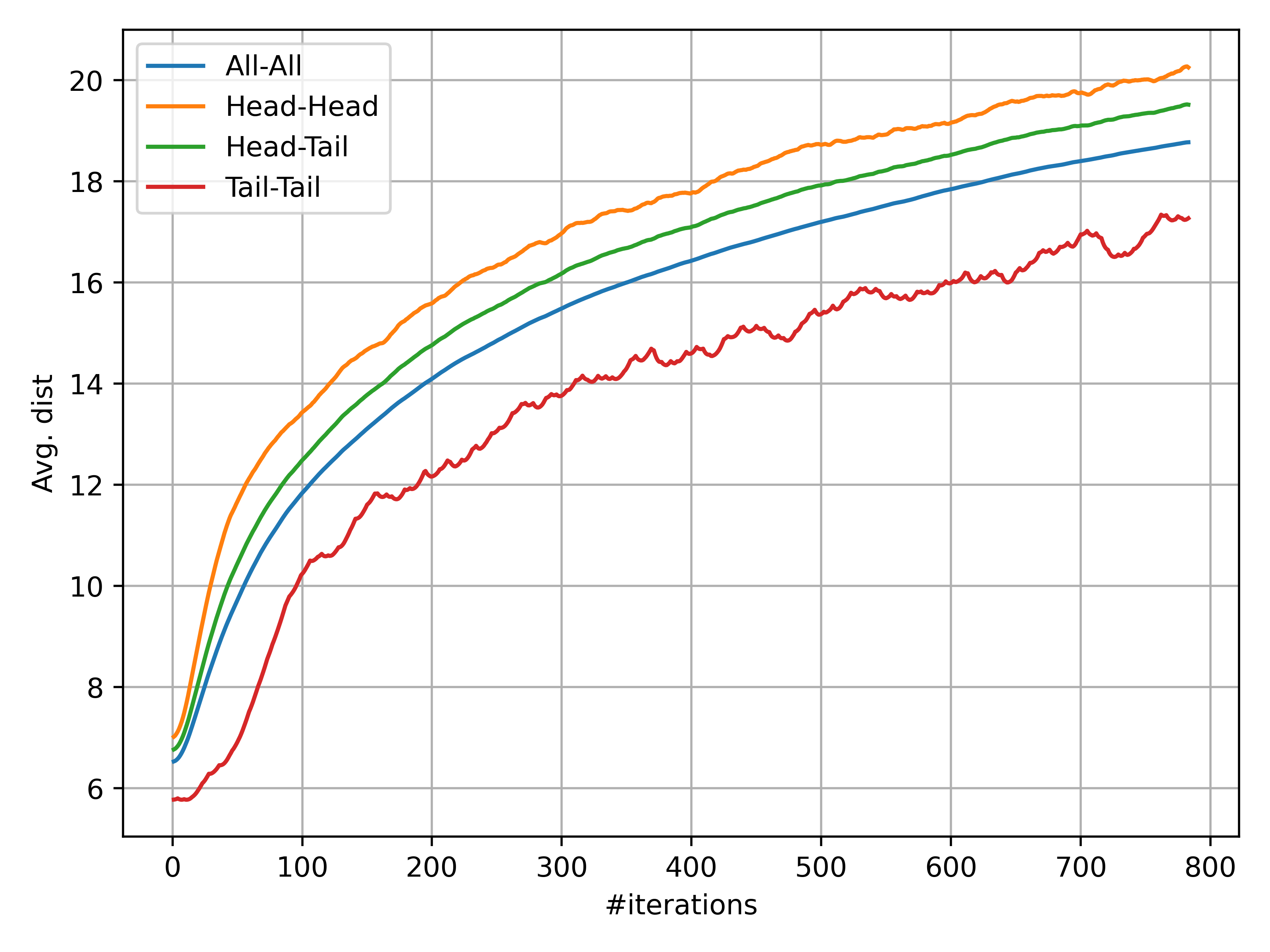}
        \caption{}
        \end{subfigure}
      \begin{subfigure}[t]{0.49\textwidth}
        \includegraphics[width=\textwidth]{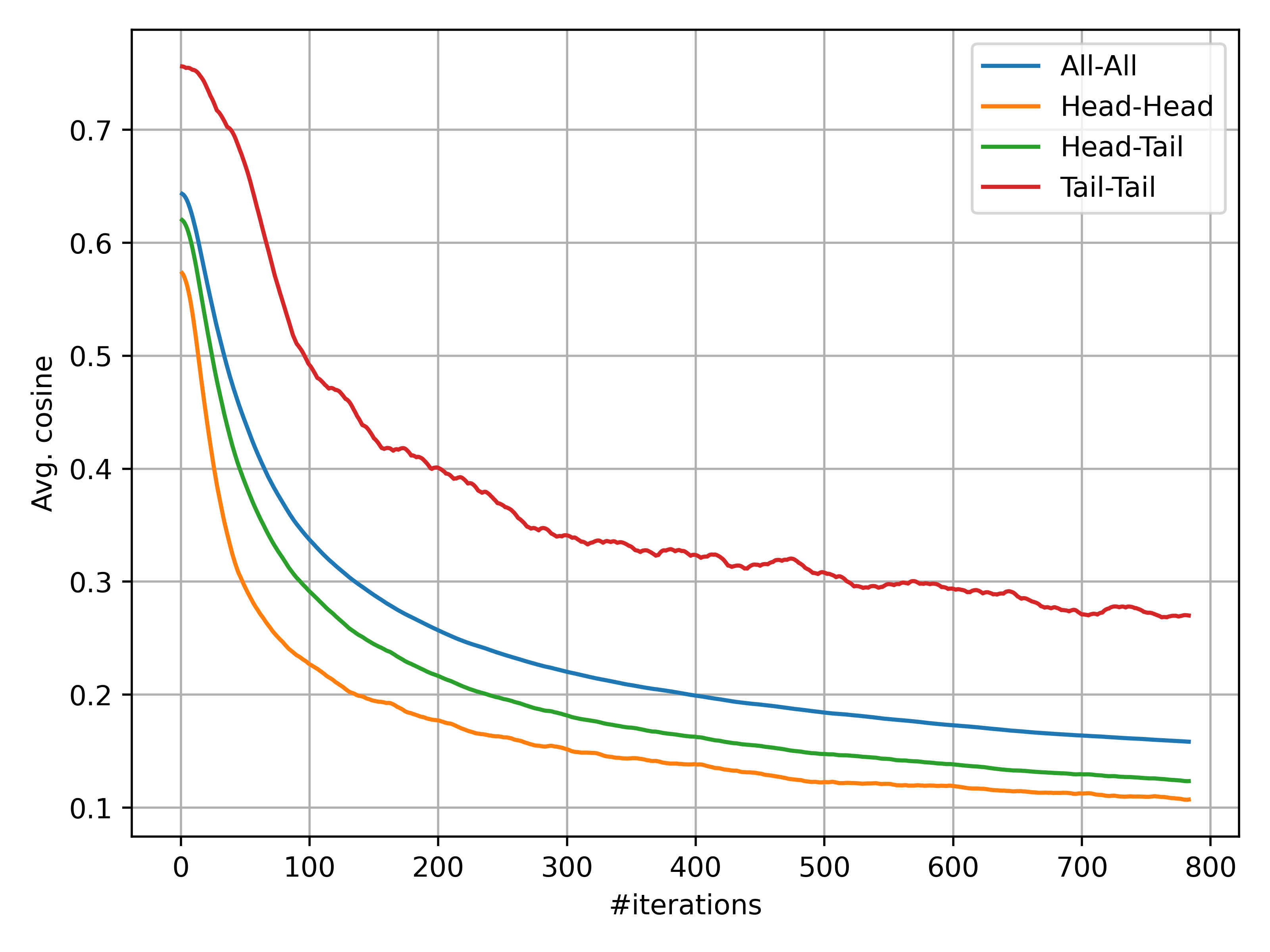} 
        \caption{}
        \end{subfigure}
   \caption{Evolution of prototypes during training. a) The average Euclidean distance between prototypes vs number of training iterations. b) The average cosine similarity between prototypes vs number of training iterations, a lower value implies higher angular separation. We compare across pairs of All-All, Head-Head, Head-Tail and Tail-Tail class prototypes. Euclidean and angular separation increase monotonically. Tail-Tail classes have smaller separation than Head-Head, which is due to the Minority Collapse effect in imbalanced training.} 
    \label{fig:fig3}
\end{figure*}

\subsection{Ablation Study}
\label{ablations}
\paragraph{Model components.} The benefits gained by adding different components to our Prototype classifier (PC) are investigated in an ablation study over CIFAR100-LT with imbalance ratio 100, shown in Table~\ref{tab:table 1}. We make several observations: (i) NCM, which is a special case of our model where the empirical centroids are used as class prototypes, i.e $c_y = \frac{1}{N_i}g(x)$, makes considerable gains above softmax, boosting \textit{All} accuracy from 46.12 to 48.67 and \textit{Few} accuracy from 11.34 to 14.13. NCM does not suffer from classifier norm correlation to class size like softmax and can do better even without any learning. (ii) Prototype classifier (PC) lifts us above NCM significantly, achieving ~4\% improvement on \textit{All} accuracy and ~10\% improvement on \textit{Few} accuracy. Thus, learning prototypes gives a far superior nearest-prototype model than simply using the class-centroids as done by previous work. (iii) Adding channel dependent temperatures (CDT) to PC further improves the results by 0.5\% on \textit{All} accuracy and ~4\% on \textit{Few} accuracy. CDT enhances the model by effecting distinct distance scales for all the features. (iv) Logit adjustment (LA) further improves results for PC. By combining both CDT and LA to PC we're able to gain ~1\% and ~9\% improvements  over the baseline PC in \textit{All} and \textit{Few} accuracies respectively. In the remainder, PC refers to our best model PC + CDT + LA. 

\begin{table}[t!]
    \centering
    \begin{tabular}{l | l l l l }
    \toprule
    Method & Many & Med & Few & All \\
    \midrule
    Softmax & 76.12 & 46.13 & 11.34 & 46.12  \\
    \midrule
    NCM & \textbf{76.54} & 50.40 & 14.13 & 48.67 \\
    PC & 74.00 & 55.17 & 24.50 & 52.56 \\
     + CDT  & 70.77 & 55.29 & 28.63 & 53.06 \\
     + LA  & 69.29 & 54.63 & 32.67 & 53.17 \\
     + CDT + LA  & 70.17 & \textbf{55.89} & \textbf{33.93} & \textbf{53.41} \\
    \bottomrule
    \end{tabular}
    \caption{Ablation study for CIFAR100-LT with imbalance ratio 100. NCM, PC, CDT and LA refer to Nearest-Class-Mean, Prototype classifier, Channel-dependent Temperature and Logit adjustment respectively. NCM alone with class centroids outperforms Softmax. PC via learning prototypes outperforms NCM by 4\% on \textit{All} accuracy. Adding CDT and LA further yields 1\% improvement in \textit{All} accuracy.}
    \label{tab:table 1}
\end{table}

\paragraph{Distance metrics.}
We experiment with different distance metrics for the Prototype classifier, particularly (a) Euclidean distance from Eq.~\ref{eq:2}, (b) Squared Euclidean distance and (c) Cosine distance: $d(g(x),c_y) = 1 - \frac{g(x)^Tc_y}{\|g(x)\|\|c_y\|}$. The results on CIFAR100-LT are depicted in Table~\ref{tab:table 2}. We make some important observations: (i) Squared euclidean distance does considerably worse, with \textit{All} accuracy dropping by more than 30\%. This agrees with our theoretical analysis that the Squared euclidean distance is highly sensitive to outliers for learning prototypes. (ii) Cosine distance underperforms Euclidean distance on \textit{All} and \textit{Tail} accuracy. This suggests that the geometry of the representation space is closer to Euclidean and not hyperspherical.

\begin{table}[t!]
    \begin{tabular}{l | l l l l }
    \toprule
    Method & Many & Med & Few & All \\
    \midrule
    Softmax &  \textbf{76.12} & 46.13 & 11.34 & 46.12  \\
    \midrule
    PC + Euclidean  & 70.17 & \textbf{55.89} & \textbf{33.93} & \textbf{53.41} \\
    PC + Squared Euclidean & 39.89 & 21.77 & 8.10 & 24.01 \\
    PC + Cosine & 72.94 & 54.80 & 24.77 & 52.14 \\
    \bottomrule
    \end{tabular}
    \caption{Effect of different distance metrics in learning prototypes on CIFAR100-LT. Euclidean distance outperforms Squared Euclidean and Cosine distance. Note the really bad performance of Squared Euclidean, which agrees with our theoretical analysis that its prone to outliers.}
    \label{tab:table 2}
\end{table}

\paragraph{Feature backbones.}
Since Prototype classifier is agnostic to the feature backbone used to generate representations, we experiment with various representation learning and classifier learning schemes from the LTR literature. Particularly, we use CE, DRO-LT~\cite{samuel2021distributional} and WD~\cite{alshammari2022long} as pre-trained feature backbones. For classifier learning, we use CE, cRT~\cite{kang2019decoupling}, WD + MaxNorm~\cite{alshammari2022long} and Prototype classifier (PC). The results on CIFAR100-LT are depicted in Table~\ref{tab:table 3}. We observe that Prototype classifier outperforms the competing classifier learning schemes for all the different feature backbones, thus showing that the performance gain is agnostic to the choice of backbone. Our best results are obtained using the WD baseline, which applies strong weight decay for regularizing feature weights. 
\begin{table}[ht]
    \begin{tabular}{p{2.4cm} | l | l | p{1.5cm} | l}
    \toprule
    & \multicolumn{4}{c}{Classifier learning}\\
    \midrule
    Representation Learning & CE & cRT & WD + MaxNorm & PC (Ours) \\
    \midrule
    CE & 38.3 & 41.3 & 43.4 & \textbf{44.1} \\
    DRO-LT & 41.2 & 47.3 & 47.8 & \textbf{48.2} \\
    WD & 46.1 & 50.1 & 53.3 & \textbf{53.4} \\
    \bottomrule
    \end{tabular}
    \caption{Effect of different representation and classifier learning schemes for CIFAR100-LT. PC outperforms competing classifier learning schemes for all different backbones.}
    \label{tab:table 3}
\end{table}

\subsection{Prototype Norms}
\label{prototype norms}
To understand the advantage of prototype classifier over softmax, we plot the norm of softmax classifier, empirical centroids and learned prototypes In Fig.~\ref{fig:fig2}. We recover the correlation of classifier norm to class size for softmax, as shown by prior work. For empirical centroids, it is interesting to see that this correlation does not hold; however the centroid norms are far from balanced. After learning the prototypes, we see that the norms become more balanced, with a slight favor towards the head classes. Since higher norm implies a negative bias in the prototype classifier, this actually favors the tail classes slightly.    

\subsection{How Do Prototypes Evolve During Training?}
\label{Section6}
We are interested in the evolution of the class prototypes as training proceeds. Particularly, we seek to know the average separation in terms of Euclidean distance, as well as the average cosine similarity between the class prototypes. In Fig.~\ref{fig:fig3}, we plot the average of these values between pairs of All-All classes, Head-Head classes, Head=Tail classes and Tail-Tail classes. The average distance increases monotonically for all the pairs. The separation between Head-Head classes is greater than Tail-Tail classes. Similarly, the average cosine similarity decreases monotonically for all the pairs, therefore the angular separation increases. There is higher angular separation between Head-head classes over Tail-Tail classes. Thus, the prototype classifier learns prototypes that are better separated than the empirical centroids, leading to lower classification error using the nearest-prototype decision function. It also mitigates Minority Collapse, wherein Minority Class means collapse towards the same point~\cite{fang2021exploring}.

\section{Comparison to State-of-the-Art}
\paragraph{Compared methods.} We compare against recently published and relevant methods in the LTR field. We choose representative models from the different solutions spaces- (i) data reweighting: Classifier retraining (cRT)~\cite{kang2019decoupling} and DisAlign~\cite{zhang2021distribution}, (ii) loss reshaping: Focal loss \cite{lin2017focal}, Label-distribution aware margin (LDAM) loss \cite{cao2019learning} and Logit-adjustment loss~\cite{menon2020long}, (iii) Ensemble-based models: RIDE~\cite{wang2020long}, (iv) Self-supervised contrastive learning: DRO-LT~\cite{samuel2021distributional} and PaCO~\cite{cui2021parametric}, and (v) Weight Regularization: WD, and WD + WD + Max~\cite{alshammari2022long}. 

\paragraph{Results.} Our results on CIFAR100-LT, ImageNet-LT and iNaturalist18 are reported in Table~\ref{tab:table 4}, Table~\ref{tab:table 5} and Table~\ref{tab:table 6} respectively. The numbers of compared methods
 are taken from their respective papers. Prototype classifier outperforms the prior state-of-the-art methods DRO-LT, DisAlign etc. on all imbalance ratios in CIFAR100. We outperform WD + WD + Max, with whom we share the weight decay tuned baseline WD, on all three benchmark datasets. For ImageNet-LT, Prototype classifier outperforms all prior state-of-the-arts on \textit{Tail} and \textit{All} accuracy except models with ``bells and whistles,'' particularly RIDE, which learns and fuses multiple models into an ensemble. A similar conclusion can be derived on iNaturalist18, where our results rival prior state-of-the-arts, except the ensemble model RIDE, and PaCO, which uses self-supervised pre-training and aggressive data augmentation techniques~\cite{cubuk2020randaugment,he2020momentum}. However, since Prototype classifier is a classifier layer that can be applied to any feature backbone, our method is complimentary to ensembles and self-supervised pre-training. We leave this exploration for future work. 

\section{Conclusion}
In this article, we proposed Prototype classifiers to address the biased softmax softmax problem for long-tailed recognition. Prototype classifiers do not suffer from the biased softmax problem of correlating classifier norm to class size. This is because they rely on distances to learnable prototypes in representation space instead of dot-product with learnable weights, which leads to a coupling of the weight and bias terms. Our theoretical analysis shows that Euclidean distance is stable for gradient-based optimization, and this is confirmed empirically. We further enable independent distance scales in the prototype distance function using channel-dependent temperatures. Our results on four benchmark long-tailed datasets show that Prototype classifier outperforms or is comparable to the prior state-of-the-art. Moreover, experiments reveal that the learned prototypes are equinorm and are much better separated in representation space than centroids..
 
\begin{table}[h!]
    \centering
    \begin{tabular}{l c c c}
    \toprule
    Imbalance ratio & 100 & 50 & 10 \\
    \midrule
    CE & 38.32 & 43.85 & 55.71 \\
    Focal Loss & 38.41 & 44.32 & 55.78 \\
    LDAM Loss & 39.60 & 44.97 & 56.91 \\
    cRT & 41.24 & 46.83 & 57.93 \\
    LogitAdjust & 42.01 & 47.03 & 57.74 \\
    DRO-LT & 47.31 & 57.57 & 63.41 \\
    PaCO & 52.00 & 56.00 & 64.20 \\
    RIDE & 49.10 & - & - \\
    WD & 46.08 & 52.71 & 66.03 \\
    WD + WD + Max & \underline{53.35} & \underline{57.71} & \underline{68.67} \\
    \midrule
    NCM & 48.67 & 53.62 & 67.82 \\
    PC (Ours) & \textbf{53.41} & \textbf{57.75} & \textbf{69.12} \\
    \bottomrule
    \end{tabular}
    \caption{Comparison to state-of-the-art on CIFAR100-LT. Prototype classifier achieves superior results across all imbalance ratios.}
    \label{tab:table 4}
\end{table}

\begin{table}[h!]
  \centering
  \begin{tabular}{l c c c c}
    \toprule
      & Many & Med & Few & All \\
    \midrule
    CE & \underline{65.9} & 37.5 & 7.7 & 44.4 \\
    Focal Loss  & 36.41 & 29.9 & 16.0 & 30.5 \\
    cRT  & 61.8 & 46.2 & 27.3 & 49.6 \\
    DRO-LT  & 64.0 & 49.8 & 33.1 & 53.5 \\
    DisAlign & 61.3 & \textbf{52.2} & 31.4 & 52.9 \\
    WD & \textbf{68.5} & 42.4 & 14.2 & 48.6 \\
    WD + WD + Max  & 62.5 & 50.4 & \underline{41.5} & \underline{53.9} \\
    \midrule
    NCM & 62.0 & 48.7 & 27.6 & 50.1 \\
    PC (Ours) & 63.5 & \underline{50.8} & \textbf{42.7} & \textbf{54.9} \\
    \midrule
    \rowcolor{lightgray} \multicolumn{5}{c}{SOTA with ``bells and whistles"} \\
    RIDE & 67.9 & 52.3 & 36.0 & 56.1 \\
    PaCO & 63.2 & 51.6 & 39.2 & 54.4 \\
    \bottomrule
  \end{tabular}
    \caption{Comparison to state-of-the-art on ImageNet-LT. We outperform all prior state-of-the-art with single models and no self-supervised pretraining or aggressive data augmentation.}
    \label{tab:table 5}
\end{table}

\begin{table}[h!]
  \centering
  \begin{tabular}{l c c c c }
    \toprule
      & Many & Med & Few & All \\
    \midrule
    CE & \underline{72.2} & 63.0 & 57.2 & 61.7 \\
    Focal Loss  & - & - & - & 61.1 \\
    cRT  & 69.0 & 66.0 & 63.2 & 65.2 \\
    DRO-LT  & - & - & - & 69.7 \\
    DisAlign & 69.0 & \textbf{71.1} & \textbf{70.2} & \textbf{70.6} \\
    WD & \textbf{74.5} & 66.5 & 61.5 & 65.4 \\
    WD + WD + Max  & 71.2 & 70.4 & \underline{69.7} & \underline{70.2} \\
    \midrule
    NCM & 61.0 & 63.5 & 63.3 & 63.1 \\
    PC (Ours) & 71.6 & \underline{70.6} & \textbf{70.2} & \textbf{70.6} \\
    \midrule
    \rowcolor{lightgray} \multicolumn{5}{c}{SOTA with ``bells and whistles"} \\
    RIDE  & 66.5 & 72.1 & 71.5 & 71.3 \\
    PaCO  & 69.5 & 72.3 & 73.1 & 72.3 \\
    \bottomrule
  \end{tabular}
    \caption{Comparison to state-of-the-art on iNaturalist18. We outperform all prior state-of-the-art with single models and no self-supervised pretraining or aggressive data augmentation.}
    \label{tab:table 6}
\end{table}

\section*{Acknowledgements}
Research was partially sponsored by the NSF under award \#2229876.

\bibliographystyle{named}
\bibliography{ijcai23}

\begin{thebibliography}{}

\bibitem[\protect\citeauthoryear{Alshammari \bgroup \em et al.\egroup
  }{2022}]{alshammari2022long}
Shaden Alshammari, Yu-Xiong Wang, Deva Ramanan, and Shu Kong.
\newblock Long-tailed recognition via weight balancing.
\newblock In {\em Proceedings of the IEEE/CVF Conference on Computer Vision and
  Pattern Recognition}, pages 6897--6907, 2022.

\bibitem[\protect\citeauthoryear{Cao \bgroup \em et al.\egroup
  }{2019}]{cao2019learning}
Kaidi Cao, Colin Wei, Adrien Gaidon, Nikos Arechiga, and Tengyu Ma.
\newblock Learning imbalanced datasets with label-distribution-aware margin
  loss.
\newblock In {\em NeurIPS}, 2019.

\bibitem[\protect\citeauthoryear{Chawla \bgroup \em et al.\egroup
  }{2002}]{chawla2002smote}
Nitesh~V Chawla, Kevin~W Bowyer, Lawrence~O Hall, and W~Philip Kegelmeyer.
\newblock Smote: synthetic minority over-sampling technique.
\newblock {\em JAIR}, 2002.

\bibitem[\protect\citeauthoryear{Chawla \bgroup \em et al.\egroup
  }{2003}]{chawla2003smoteboost}
Nitesh~V Chawla, Aleksandar Lazarevic, Lawrence~O Hall, and Kevin~W Bowyer.
\newblock Smoteboost: Improving prediction of the minority class in boosting.
\newblock In {\em European conference on principles of data mining and
  knowledge discovery}, 2003.

\bibitem[\protect\citeauthoryear{Cubuk \bgroup \em et al.\egroup
  }{2018}]{cubuk2018autoaugment}
Ekin~D Cubuk, Barret Zoph, Dandelion Mane, Vijay Vasudevan, and Quoc~V Le.
\newblock Autoaugment: Learning augmentation policies from data.
\newblock {\em arXiv preprint arXiv:1805.09501}, 2018.

\bibitem[\protect\citeauthoryear{Cubuk \bgroup \em et al.\egroup
  }{2020}]{cubuk2020randaugment}
Ekin~D Cubuk, Barret Zoph, Jonathon Shlens, and Quoc~V Le.
\newblock Randaugment: Practical automated data augmentation with a reduced
  search space.
\newblock In {\em Proceedings of the IEEE/CVF conference on computer vision and
  pattern recognition workshops}, pages 702--703, 2020.

\bibitem[\protect\citeauthoryear{Cui \bgroup \em et al.\egroup
  }{2019}]{cui2019class}
Yin Cui, Menglin Jia, Tsung-Yi Lin, Yang Song, and Serge Belongie.
\newblock Class-balanced loss based on effective number of samples.
\newblock In {\em CVPR}, 2019.

\bibitem[\protect\citeauthoryear{Cui \bgroup \em et al.\egroup
  }{2021}]{cui2021parametric}
Jiequan Cui, Zhisheng Zhong, Shu Liu, Bei Yu, and Jiaya Jia.
\newblock Parametric contrastive learning.
\newblock In {\em Proceedings of the IEEE/CVF international conference on
  computer vision}, pages 715--724, 2021.

\bibitem[\protect\citeauthoryear{Deng \bgroup \em et al.\egroup
  }{2009}]{deng2009imagenet}
Jia Deng, Wei Dong, Richard Socher, Li-Jia Li, Kai Li, and Li~Fei-Fei.
\newblock Imagenet: A large-scale hierarchical image database.
\newblock In {\em CVPR}, 2009.

\bibitem[\protect\citeauthoryear{DeVries and Taylor}{2017}]{devries2017cutout}
Terrance DeVries and Graham~W Taylor.
\newblock Improved regularization of convolutional neural networks with cutout.
\newblock {\em arXiv preprint arXiv:1708.04552}, 2017.

\bibitem[\protect\citeauthoryear{Drummond \bgroup \em et al.\egroup
  }{2003}]{drummond2003c4}
Chris Drummond, Robert~C Holte, et~al.
\newblock C4. 5, class imbalance, and cost sensitivity: why under-sampling
  beats over-sampling.
\newblock In {\em Workshop on learning from imbalanced datasets II}, 2003.

\bibitem[\protect\citeauthoryear{Estabrooks \bgroup \em et al.\egroup
  }{2004}]{estabrooks2004multiple}
Andrew Estabrooks, Taeho Jo, and Nathalie Japkowicz.
\newblock A multiple resampling method for learning from imbalanced data sets.
\newblock {\em Computational intelligence}, 2004.

\bibitem[\protect\citeauthoryear{Fang \bgroup \em et al.\egroup
  }{2021}]{fang2021exploring}
Cong Fang, Hangfeng He, Qi~Long, and Weijie~J Su.
\newblock Exploring deep neural networks via layer-peeled model: Minority
  collapse in imbalanced training.
\newblock {\em Proceedings of the National Academy of Sciences},
  118(43):e2103091118, 2021.

\bibitem[\protect\citeauthoryear{Guerriero \bgroup \em et al.\egroup
  }{2018}]{guerriero2018deepncm}
Samantha Guerriero, Barbara Caputo, and Thomas Mensink.
\newblock Deepncm: Deep nearest class mean classifiers.
\newblock In {\em International Conference on Learning Representations
  Workshop}, 2018.

\bibitem[\protect\citeauthoryear{He and Garcia}{2009}]{he2009learning}
Haibo He and Edwardo~A Garcia.
\newblock Learning from imbalanced data.
\newblock {\em TKDE}, 2009.

\bibitem[\protect\citeauthoryear{He \bgroup \em et al.\egroup
  }{2020}]{he2020momentum}
Kaiming He, Haoqi Fan, Yuxin Wu, Saining Xie, and Ross Girshick.
\newblock Momentum contrast for unsupervised visual representation learning.
\newblock In {\em Proceedings of the IEEE/CVF conference on computer vision and
  pattern recognition}, pages 9729--9738, 2020.

\bibitem[\protect\citeauthoryear{Kang \bgroup \em et al.\egroup
  }{2020}]{kang2019decoupling}
Bingyi Kang, Saining Xie, Marcus Rohrbach, Zhicheng Yan, Albert Gordo, Jiashi
  Feng, and Yannis Kalantidis.
\newblock Decoupling representation and classifier for long-tailed recognition.
\newblock In {\em Eighth International Conference on Learning Representations
  (ICLR)}, 2020.

\bibitem[\protect\citeauthoryear{Lin \bgroup \em et al.\egroup
  }{2017}]{lin2017focal}
Tsung-Yi Lin, Priya Goyal, Ross Girshick, Kaiming He, and Piotr Doll{\'a}r.
\newblock Focal loss for dense object detection.
\newblock In {\em ICCV}, 2017.

\bibitem[\protect\citeauthoryear{Liu \bgroup \em et al.\egroup
  }{2019}]{liu2019large}
Ziwei Liu, Zhongqi Miao, Xiaohang Zhan, Jiayun Wang, Boqing Gong, and Stella~X
  Yu.
\newblock Large-scale long-tailed recognition in an open world.
\newblock In {\em CVPR}, 2019.

\bibitem[\protect\citeauthoryear{Liu \bgroup \em et al.\egroup
  }{2021}]{liu2021gistnet}
Bo~Liu, Haoxiang Li, Hao Kang, Gang Hua, and Nuno Vasconcelos.
\newblock Gistnet: a geometric structure transfer network for long-tailed
  recognition.
\newblock In {\em Proceedings of the IEEE/CVF International Conference on
  Computer Vision}, pages 8209--8218, 2021.

\bibitem[\protect\citeauthoryear{Loshchilov and Hutter}{2017}]{loshchilovsgdr}
Ilya Loshchilov and Frank Hutter.
\newblock Sgdr: Stochastic gradient descent with warm restarts.
\newblock In {\em International Conference on Learning Representations}, 2017.

\bibitem[\protect\citeauthoryear{Menon \bgroup \em et al.\egroup
  }{2021}]{menon2020long}
Aditya~Krishna Menon, Sadeep Jayasumana, Ankit~Singh Rawat, Himanshu Jain,
  Andreas Veit, and Sanjiv Kumar.
\newblock Long-tail learning via logit adjustment.
\newblock In {\em International Conference on Learning Representations}, 2021.

\bibitem[\protect\citeauthoryear{Mensink \bgroup \em et al.\egroup
  }{2013}]{mensink2013distance}
Thomas Mensink, Jakob Verbeek, Florent Perronnin, and Gabriela Csurka.
\newblock Distance-based image classification: Generalizing to new classes at
  near-zero cost.
\newblock {\em TPAMI}, 2013.

\bibitem[\protect\citeauthoryear{Nogueira and
  others}{2014}]{nogueira2014bayesian}
Fernando Nogueira et~al.
\newblock Bayesian optimization: Open source constrained global optimization
  tool for python.
\newblock {\em URL https://github. com/fmfn/BayesianOptimization}, 2014.

\bibitem[\protect\citeauthoryear{Papyan \bgroup \em et al.\egroup
  }{2020}]{papyan2020prevalence}
Vardan Papyan, XY~Han, and David~L Donoho.
\newblock Prevalence of neural collapse during the terminal phase of deep
  learning training.
\newblock {\em Proceedings of the National Academy of Sciences},
  117(40):24652--24663, 2020.

\bibitem[\protect\citeauthoryear{Rebuffi \bgroup \em et al.\egroup
  }{2017}]{rebuffi2017icarl}
Sylvestre-Alvise Rebuffi, Alexander Kolesnikov, Georg Sperl, and Christoph~H
  Lampert.
\newblock icarl: Incremental classifier and representation learning.
\newblock In {\em Proceedings of the IEEE conference on Computer Vision and
  Pattern Recognition}, 2017.

\bibitem[\protect\citeauthoryear{Ren \bgroup \em et al.\egroup
  }{2020}]{ren2020balanced}
Jiawei Ren, Cunjun Yu, Xiao Ma, Haiyu Zhao, Shuai Yi, et~al.
\newblock Balanced meta-softmax for long-tailed visual recognition.
\newblock {\em Advances in Neural Information Processing Systems},
  33:4175--4186, 2020.

\bibitem[\protect\citeauthoryear{Samuel and
  Chechik}{2021}]{samuel2021distributional}
Dvir Samuel and Gal Chechik.
\newblock Distributional robustness loss for long-tail learning.
\newblock In {\em Proceedings of the IEEE/CVF International Conference on
  Computer Vision}, 2021.

\bibitem[\protect\citeauthoryear{Sharma \bgroup \em et al.\egroup
  }{2020}]{sharma2020long}
Saurabh Sharma, Ning Yu, Mario Fritz, and Bernt Schiele.
\newblock Long-tailed recognition using class-balanced experts.
\newblock In {\em DAGM German Conference on Pattern Recognition}. Springer,
  2020.

\bibitem[\protect\citeauthoryear{Snell \bgroup \em et al.\egroup
  }{2017}]{snell2017prototypical}
Jake Snell, Kevin Swersky, and Richard Zemel.
\newblock Prototypical networks for few-shot learning.
\newblock In {\em NeurIPS}, 2017.

\bibitem[\protect\citeauthoryear{Sun \bgroup \em et al.\egroup
  }{2019}]{sun2019meta}
Qianru Sun, Yaoyao Liu, Tat-Seng Chua, and Bernt Schiele.
\newblock Meta-transfer learning for few-shot learning.
\newblock In {\em CVPR}, 2019.

\bibitem[\protect\citeauthoryear{Thrampoulidis \bgroup \em et al.\egroup
  }{2022}]{thrampoulidisimbalance}
Christos Thrampoulidis, Ganesh~Ramachandra Kini, Vala Vakilian, and Tina
  Behnia.
\newblock Imbalance trouble: Revisiting neural-collapse geometry.
\newblock In {\em Advances in Neural Information Processing Systems}, 2022.

\bibitem[\protect\citeauthoryear{Van~Horn and Perona}{2017}]{van2017devil}
Grant Van~Horn and Pietro Perona.
\newblock The devil is in the tails: Fine-grained classification in the wild.
\newblock {\em arXiv preprint arXiv:1709.01450}, 2017.

\bibitem[\protect\citeauthoryear{Van~Horn \bgroup \em et al.\egroup
  }{2018}]{van2018inaturalist}
Grant Van~Horn, Oisin Mac~Aodha, Yang Song, Yin Cui, Chen Sun, Alex Shepard,
  Hartwig Adam, Pietro Perona, and Serge Belongie.
\newblock The inaturalist species classification and detection dataset.
\newblock In {\em Proceedings of the IEEE conference on computer vision and
  pattern recognition}, pages 8769--8778, 2018.

\bibitem[\protect\citeauthoryear{Wang \bgroup \em et al.\egroup
  }{2020}]{wang2020long}
Xudong Wang, Long Lian, Zhongqi Miao, Ziwei Liu, and Stella Yu.
\newblock Long-tailed recognition by routing diverse distribution-aware
  experts.
\newblock In {\em International Conference on Learning Representations}, 2020.

\bibitem[\protect\citeauthoryear{Xian \bgroup \em et al.\egroup
  }{2019}]{xian2019f}
Yongqin Xian, Saurabh Sharma, Bernt Schiele, and Zeynep Akata.
\newblock f-vaegan-d2: A feature generating framework for any-shot learning.
\newblock In {\em CVPR}, 2019.

\bibitem[\protect\citeauthoryear{Xiang \bgroup \em et al.\egroup
  }{2020}]{xiang2020learning}
Liuyu Xiang, Guiguang Ding, and Jungong Han.
\newblock Learning from multiple experts: Self-paced knowledge distillation for
  long-tailed classification.
\newblock In {\em European Conference on Computer Vision}. Springer, 2020.

\bibitem[\protect\citeauthoryear{Zhang \bgroup \em et al.\egroup
  }{2021}]{zhang2021distribution}
Songyang Zhang, Zeming Li, Shipeng Yan, Xuming He, and Jian Sun.
\newblock Distribution alignment: A unified framework for long-tail visual
  recognition.
\newblock In {\em Proceedings of the IEEE/CVF conference on computer vision and
  pattern recognition}, pages 2361--2370, 2021.

\bibitem[\protect\citeauthoryear{Zhou \bgroup \em et al.\egroup
  }{2020}]{zhou2020bbn}
Boyan Zhou, Quan Cui, Xiu-Shen Wei, and Zhao-Min Chen.
\newblock Bbn: Bilateral-branch network with cumulative learning for
  long-tailed visual recognition.
\newblock In {\em Proceedings of the IEEE/CVF conference on computer vision and
  pattern recognition}, pages 9719--9728, 2020.

\end{thebibliography}

\clearpage
\appendix

\Large{\textbf{Appendix}}

\begin{figure*}[!ht]
    \centering
    \begin{subfigure}[t]{0.33\textwidth}
        \includegraphics[width=\textwidth]{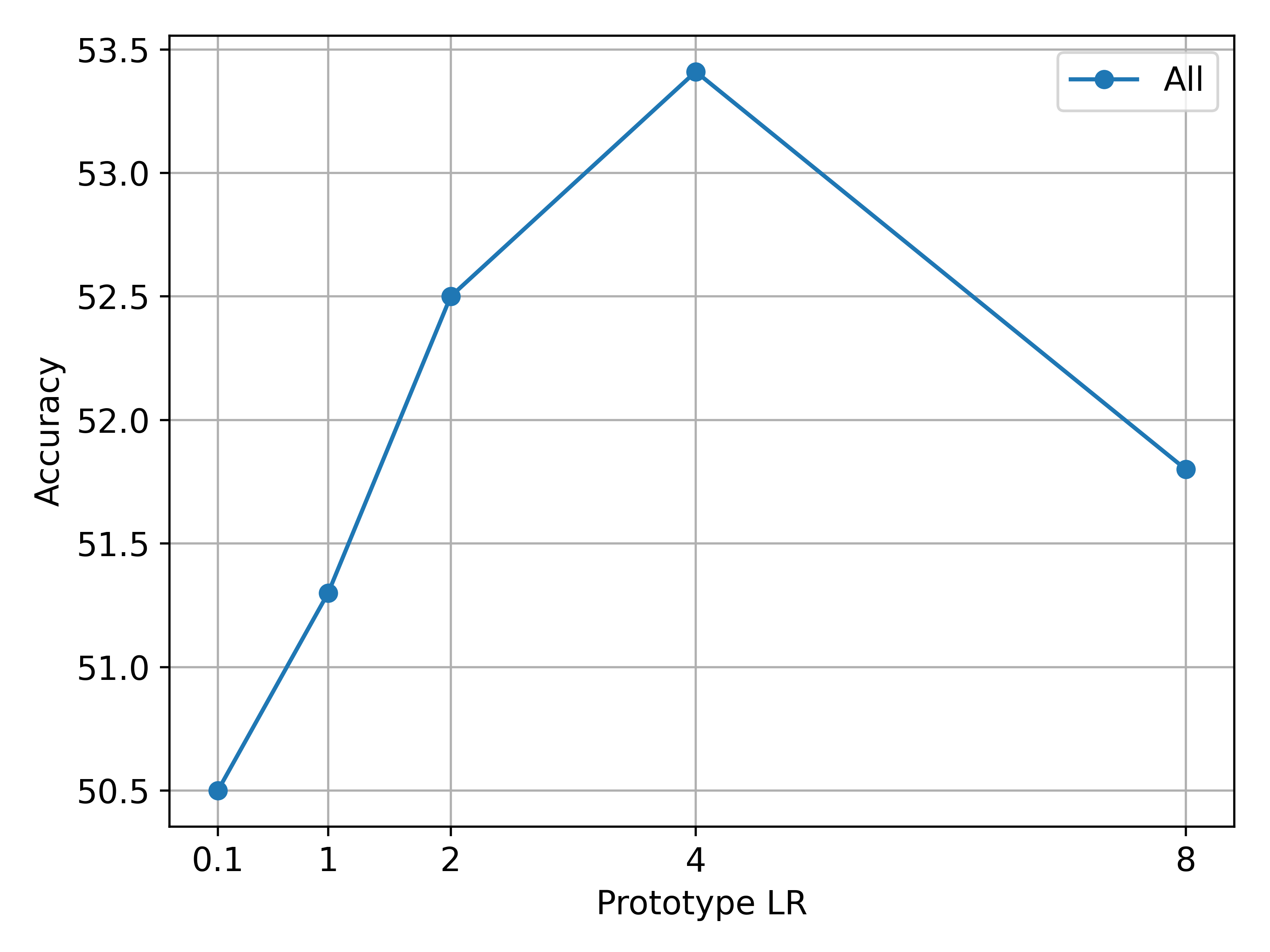}
        \caption{}
        \label{fig:4a}
        \end{subfigure}
      \begin{subfigure}[t]{0.33\textwidth}
        \includegraphics[width=\textwidth]{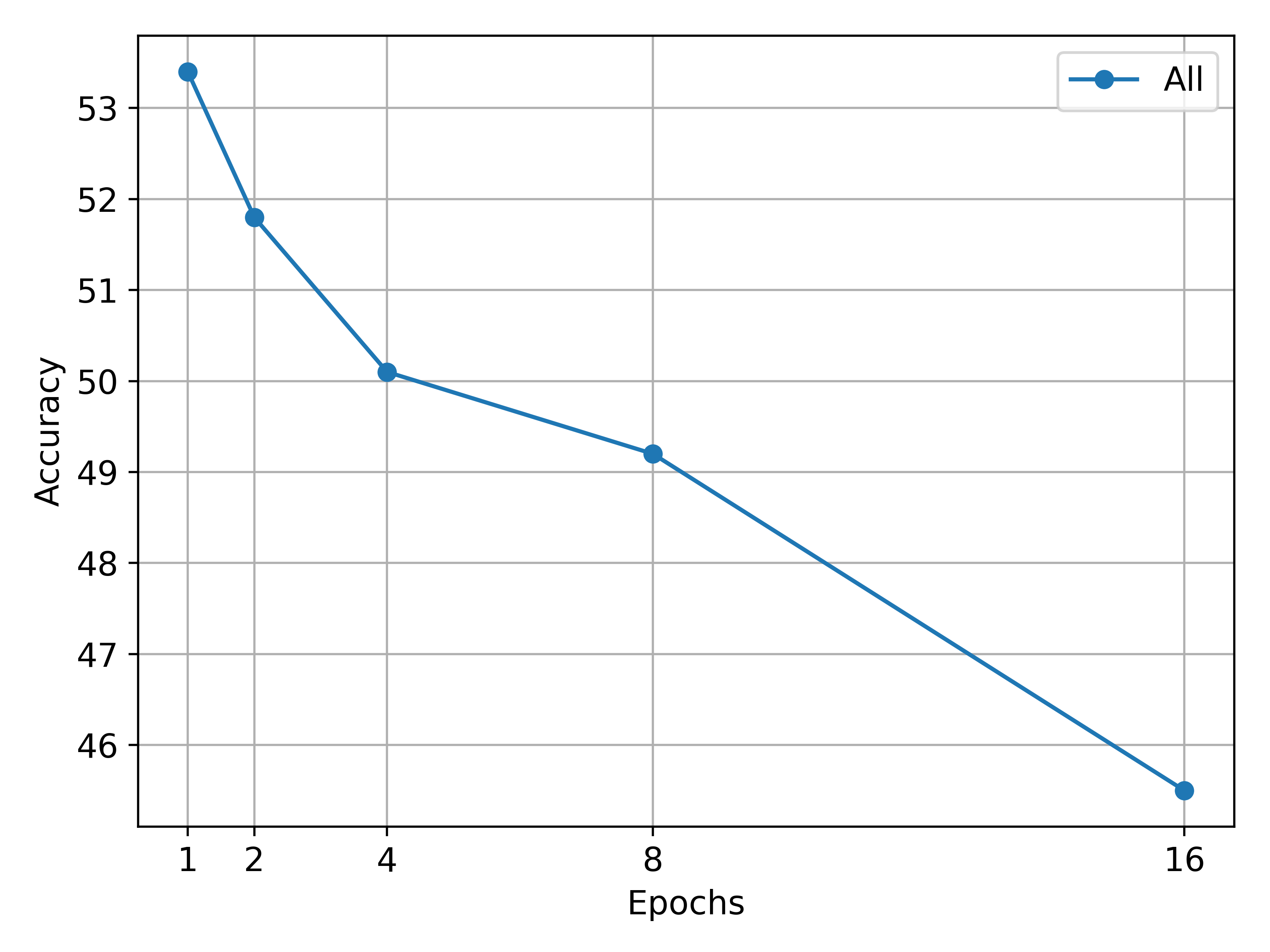}
        \caption{}
        \label{fig:4b}
        \end{subfigure}
      \begin{subfigure}[t]{0.33\textwidth}
        \includegraphics[width=\textwidth]{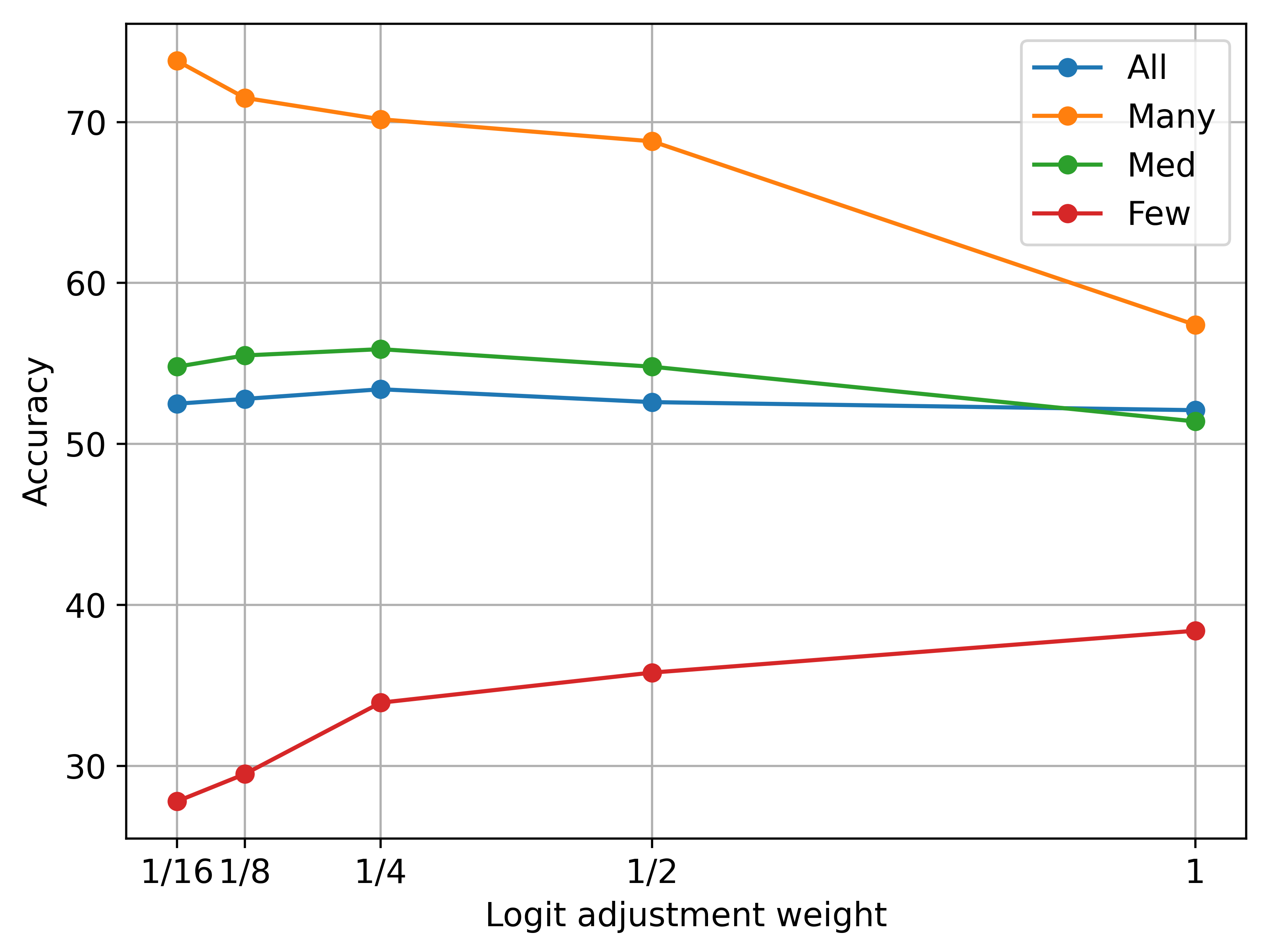}
        \caption{}
        \label{fig:4c}
        \end{subfigure}
   \caption{L-R: a) Effect of the prototype learning rate on CIFAR100-LT. b) Effect of the number of training epochs on CIFAR100-LT. c) Effect of the logit-adjustment weight $\tau$ on CIFAR100-LT.} 
\end{figure*}

\section{Dataset statistics}
We detail the dataset statistics for the three benchmark long-tailed recognition datasets in Table~\ref{tab:table 7}.
\begin{table}[h]
\small
\centering
\begin{tabular}{ p{2cm} | p{1cm} | p{0.8cm} p{0.8cm} p{0.8cm} p{1cm} }
\toprule
Dataset & Attribute & Many & Medium & Few & All \\
 \midrule
\multirow{2}{2cm}{CIFAR10-LT (Imba 100)} & Classes & 8 & 2 & 0 & 10 \\
 & Samples & 12273 & 133 & 0 & 12406 \\
 \midrule
 \multirow{2}{2cm}{CIFAR100-LT (Imba 100)} & Classes & 35 & 35 & 30 & 100 \\
 & Samples & 8824 & 1718 & 305 & 10847 \\
 \midrule
 \multirow{2}{2cm}{CIFAR100-LT (Imba 50)} & Classes & 41 & 40 & 19 & 100 \\
 & Samples & 10331 & 2008 & 269 & 12608 \\
 \midrule
 \multirow{2}{2cm}{CIFAR100-LT (Imba 10)} & Classes & 70 & 30 & 0 & 100 \\
 & Samples & 17743 & 2130 & 0 & 19573 \\
\midrule
\multirow{2}{2cm}{ImageNet-LT} & Classes & 391 & 473 & 136 & 1,000 \\
 & Samples & 89,293 & 24,910 & 1,643 & 115,846 \\
 \midrule
\multirow{2}{2cm}{iNaturalist18} & Classes & 842 & 4076 & 3599 & 8,142 \\
 & Samples & 433,806 & 258,340 & 133,061 & 437,513 \\
\bottomrule
\end{tabular}
\caption{Statistics for training data in CIFAR10-LT, CIFAR100-LT, ImageNet-LT and iNaturalist18.}
\label{tab:table 7}
\end{table}

\section{More training details}
We mention here a few more pertinent training details. The prototype learning is most optimal with high prototype learning rate. The effect of the prototype learning rate is shown in Fig.~\ref{fig:4a}. Further, we find that training for 1 epoch is enough. The effect of training epochs is shown in Fig.~\ref{fig:4b}. Training beyond 1 epoch leads to lower validation accuracy. Thus, the prototype learning has very fast convergence and is quick to train. 

\section{Temperature schemes}
We discuss here some other temperature schemes for learning prototypes. Apart from Channel-dependent temperatures, we also experiment with Class-dependent temperatures. We vary the temperature according to class, allowing for different distance scales specific to the local neighborhood of the class. Specifically, we only allow a per-class temperature $T_y$:
\begin{align}
\label{eq:6}
    d_{Class}(g(x),y) &= \sqrt{\frac{\sum\limits_{i=1}^d(g(x) - c_y)_i^2}{T_y}}
\end{align}
Secondly, we also experiment with Dense temperatures, allowing different temperature scales per-class as well as per-channel $T_{yi}$:
\begin{align}
\label{eq:7}
    d_{Dense}(g(x),y) &= \sqrt{\sum\limits_{i=1}^d\frac{(g(x) - c_y)_i^2}{T_{yi}}}
\end{align}
The three schemes, Channel temps, Class temps and Dense temps are compared in Table~\ref{tab:table 8}. We find that both Class temps and Dense temps have similar results and underperform Channel temps. We hypothesize this is because Channel-temps encourage knowledge transfer from Many to Few shot classes by enabling globally shared feature selection.

\begin{table}[ht]
    \centering
    \begin{tabular}{l | l l l l }
    \toprule
    Method & Many & Med & Few & All \\
    \midrule
    PC & \textbf{74.00} & 55.17 & 24.50 & 52.56 \\
     + Channel T  & 70.77 & \textbf{55.29} & \textbf{28.63} & \textbf{53.06} \\
     + Class T  & 72.33 & 54.83 & 25.62 & 52.47 \\
     + Dense T  & 71.28 & 54.77 & 24.11 & 52.17 \\
    \bottomrule
    \end{tabular}
    \caption{Effect of different temperature schemes learning prototypes, i.e Channel-dependent Temperature, Class-Dependent and Dense temperatures respectively, on CIFAR100-LT. Channel-dependent does the best.}
    \label{tab:table 8}
\end{table}

\section{Effect of logit adjustment weight}
In Fig.~\ref{fig:4c}, we study the effect of the logit adjustment weight on the different performance metrics. We observe a clear trade-off between \textit{Few} accuracy on one hand and \textit{All}, \textit{Many} and \textit{Med} on the other, with higher weight favoring \textit{Few}. Consider that the logit adjustment $\log{N_y}$ leads to a higher loss on tail classes during the optimization process. Therefore, increasing the weight parameter leads to tail loss being over-emphasized during training and causing lower tail class error at the price of higher head class error. We choose $\tau=1/4$ as the optimal weight achieving good accuracy across the spectrum.

\section{Results on CIFAR10-LT}
The results for CIFAR10-LT with imbalance ratio 100 are depicted in Table~\ref{tab:table 9}. The results are similar to CIFAR100-LT, we surpass the competing state-of-the-art methods cRT~\cite{kang2019decoupling}, LDAM~\cite{cao2019learning}, DRO-LT~\cite{samuel2021distributional}, and WD + WD + Max~\cite{alshammari2022long}. 

\begin{table}[h!]
  \centering
  \begin{tabular}{l c c c}
    \toprule
      & All \\
    \midrule
    CE & 72.1 \\
    cRT & 73.5 \\
    LDAM & 77.03 \\
    DRO-LT & 82.6 \\
    WD & 78.6 \\
    WD + WD + Max & \underline{82.9} \\
    \midrule
    NCM & 80.1 \\
    PC (Ours) & \textbf{83.7} \\
    \bottomrule
  \end{tabular}
   \caption{Comparison to state-of-the-art on CIFAR10-LT with imbalance ratio 100. Prototype classifier achieves superior results.}
  \label{tab:table 9}
\end{table}

\end{document}